\documentclass{article}

\PassOptionsToPackage{numbers}{natbib}
     \usepackage[preprint]{neurips_2020}

\usepackage[utf8]{inputenc} 
\usepackage[T1]{fontenc}    
\usepackage{hyperref}       
\usepackage{url}            
\usepackage{booktabs}       
\usepackage{amsfonts}       
\usepackage{nicefrac}       
\usepackage{microtype}      

\usepackage{mathtools}
\usepackage{amsthm}
\usepackage{xr}
\usepackage{wrapfig}
\usepackage{subfigure} 
\usepackage{caption}

\newtheorem{proposition}{Proposition}

\title{Splitting Gaussian Process Regression for\\ Streaming Data}

\author{%
  Nick Terry\\
  Department of Industrial and Systems Engineering\\
  University of Washington\\
  Seattle, WA 98105 \\
  \texttt{pnterry@uw.edu} \\
  \And
  Youngjun Choe\\
  Department of Industrial and Systems Engineering\\
  University of Washington\\
  Seattle, WA 98105 \\
  \texttt{ychoe@uw.edu} \\
}

\begin{document}

\maketitle

\begin{abstract}
  Gaussian processes offer a flexible kernel method for regression. While Gaussian processes have many useful theoretical properties and have proven practically useful, they suffer from poor scaling in the number of observations. In particular, the cubic time complexity of updating standard Gaussian process models make them generally unsuitable for application to streaming data. We propose an algorithm for sequentially partitioning the input space and fitting a localized Gaussian process to each disjoint region. The algorithm is shown to have superior time and space complexity to existing methods, and its sequential nature permits application to streaming data. The algorithm constructs a model for which the time complexity of updating is tightly bounded above by a pre-specified parameter. To the best of our knowledge, the model is the first local Gaussian process regression model to achieve linear memory complexity. Theoretical continuity properties of the model are proven. We demonstrate the efficacy of the resulting model on multi-dimensional regression tasks for streaming data. 
\end{abstract}

\section{Introduction}

Gaussian process (GP) regression is a flexible kernel method for approximating smooth functions from data. Assuming there is a \textit{latent function} which describes the relationship between predictors and a response, from a Bayesian perspective a GP defines a prior over latent functions. When conditioned on the observed data, the GP then gives a posterior distribution of the latent function. In practice, GP models are fit to training data by optimizing the marginal log likelihood with respect to hyperparameters of the \textit{kernel function}, $k$.

A notable shortcoming of GP regression is its poor scaling both in memory and computational complexity. In what we will hereafter refer to as \textit{full GP regression}, the posterior mean and variance can be directly computed at any finite number of points using linear algebra \citep{Gaussian_Processes_for_ML}. However, a computational bottleneck is caused by the necessary matrix inversion of the kernel matrix $\mathbf{K}$, which is well-known to have a time complexity of $\mathcal{O}(n^3)$, where $n$ is the number of observations. The $\mathcal{O}(n^2)$  memory complexity of storing $\mathbf{K}$ may also pose issues for massive data sets.

GP regression may also struggle to well-approximate latent functions which are \textit{non-stationary} \citep{gramacy_bayesian_2005,gramacy_bayesian_2008}. Non-stationary latent function's mean or covariance may vary over its domain. Non-stationarity may be induced in more subtle ways as well, such as heteroscedastic additive noise in the observed response.

Local GP regression is a class of models which address both of these problems, to varying extents. The commonality of these models is the assignment of observations to one of many \textit{local} GPs during training, and the aggregation of the local models' individual predictions. As a result of observations being assigned to a single local model, effectively only a block-diagonal approximation of the full kernel matrix $\mathbf{K}$ is maintained \citep{park_patchwork_2018,blockgp}, easing the time and memory complexity. The price of this flexibility and computational advantage is a potential decrease in predictive ability, relative to full GP regression, on tasks for which a full GP is well-suited.

A successful method for assigning observations to local GPs is partitioning the \textit{input space}, the domain of the latent function, and creating a local GP model for each cell of the partition. Existing local GP models may encounter various difficulties during the partitioning process, as detailed in Section~\ref{sec:related_work}, and/or
neglect a setting where only a sequential data source is available and partitioning must be performed sequentially.

This sequential setting is important since it encompasses tasks with changing dynamics. For example, 
in applications such as process quality monitoring \citep{yu_online_2012,liu_auto-switch_2015} and motion tracking \citep{5206580}, a sequential approach allows for the model to adapt to regimes which are yet-unobserved.

The primary contribution of this paper is an algorithm which recursively partitions the input space to construct local GPs using sequential observations. The resulting model is dubbed the \textit{splitting GP}. The algorithm is shown to have superior asymptotic time and memory complexity relative to other state-of-the-art local GP methods which can be used in this problem setting. By design of the algorithm, we also ensure an exact upper bound on the time complexity of updating the splitting GP model. We also prove theoretical properties of the model related to continuity of the predictions and empirically demonstrate the efficacy of the model. Additionally, a software implementation of the algorithm is provided, which leverages the computational advantages of the \texttt{GPyTorch} library \citep{gardner_gpytorch_2018}, to facilitate the use of the algorithm by others. 

\section{Related work}\label{sec:related_work}

It appears that the first exploration of local GP regression is due to \citet{rasmussen_infinite_2002}, as a special case of their mixture of GP experts model, where prediction at a point is performed by one designated ``expert'' GP. \citet{snelson_local_2007} further developed this idea, but only as a supplementary method to sparse GP regression. 

An adjacent line of research on treed GP models was begun by \citet{gramacy_bayesian_2005}. Treed GP models perform domain decomposition in the same manner as classification and regression trees \citep{breiman_classification_1984}, and fit distinct GP models to each resulting partition. Predictions are then formed using $k$-nearest neighbors. \citet{gramacy_bayesian_2005} found that local GP regression methods, such as treed GP models, were well-suited to non-stationary regression tasks since each leaf GP model in the tree could specialize to local phenomena in the data. Since the inception of this method, further advances have been made by Gramacy and collaborators \citep{gramacy_bayesian_2008,gramacy_local_2015}, particularly with respect to large scale computer experiments and surrogate modeling. It is acknowlegded by \citet{gramacy_bayesian_2005} and \citet{gramacy_bayesian_2008} that predictions by the treed GP model are not continuous in the input space. Additionally, the construction of the tree is performed probabilistically, and does not make use of the intrinsic structure of the data for domain decomposition.

Since the original proposal of local GP regression, several methods have been proposed which are adapted to the sequential setting. \citet{shen2006fast} reduced the prediction time and kernel matrix storage for isotropic kernels by building a kd-tree from the training data set.  \citet{nguyen-tuong_local_2009} proposed a method of local GP regression for online learning, which assigns incoming data to local models by similarity in the feature space and forms mean predictions by weighting local predictions. This \textit{local GP model} was among the first to consider a fully sequential setting. This model has two notable drawbacks: it suffers from discontinuities in its predictive mean, and depends on a hyperparameter which is difficult to tune in practice and strongly affects prediction performance. Another local GP method which may be used in the sequential setting is the robust Bayesian committee machine (rBCM) by \citet{deisenroth_distributed_2015}, which can be seen as a product of GP experts \citep{hinton_training_2002}. The rBCM emphasizes rapid, distributed computation over model flexibility. This is demonstrated by a) its assumption that the latent function is well-modeled by a single GP and b) consequent random assignment of observations to GP experts, rather than a partitioning-based approach. This modeling approach does not address potential non-stationarity of the latent function.

More recently, Park and collaborators have done significant work in this area, applying mesh generating procedures from finite element analysis \citep{park_efficient_2016,park_domain_2011} and the recursive Principal Direction Divisive Partitioning algorithm \citep{park_patchwork_2018,boley_principal_1998} to partition the input space for fitting local models. However, in these papers it is assumed that a substantial number of observations are available during the initial model construction to perform the partitioning procedure, particularly when using mesh generation methods. These models also suffer from  discontinuities at the boundaries of partition cells, an issue which the authors have creatively addressed by adding constraints to the hyperparameter optimization which force equality at finitely many boundary points, or by adding ``pseudo-inputs'' at the boundaries to induce continuity.

\section{Recursive splitting of local Gaussian processes}

As previous works have shown, there is a trade-off between the predictive ability of the aggregate model, the number of local models, and computation speed. Given this fact, we aim to construct a model which maintains strong predictive capability while keeping its computational demands below a pre-specified upper bound. We show that this can be accomplished in a straightforward manner by recursively \textit{splitting} GP models once they surpass a presupposed threshold in the number of observations. Splitting the model amounts to performing a clustering subroutine which divides the observations associated with the model into two subsets, and then fitting a new local model to each subset.

We consider streaming data that requires the model to permit sequential updating and prediction. A natural quantity of interest is the time \(\tau\) required for a single update; since \(\tau\) is a function of the size, \(m \times m\), of the kernel matrix \(\mathbf{K}_i\) of a local model indexed by $i$, \(m\) is then an interpretable parameter describing the period of splitting. The parameter \(m\) may be interpreted as the \textit{splitting limit}, the maximum number of observations which may be assigned to a local model before it is split. For the remainder of the paper, we will describe the splitting GP model, which is characterized almost entirely by this intuitive parameterization. Additionally, the full specification of the splitting GP algorithm may be found as pseudo-code in the appendix.

\subsection{Notation}
In preparation for the proceeding material, we define some notation. We assume some familiarity of readers with the theory of GPs and kernel methods \citep{Gaussian_Processes_for_ML}. 

The input data matrix and the response vector associated with the $i^{th}$ local model are denoted by \\$X_i=[\mathbf{x}_{i1} \, \mathbf{x}_{i2} \,\ldots \,\mathbf{x}_{in_i}]^T\in \mathbb{R}^{n_i\times M}$ and $Y_i=\left(y_{i1},y_{i2},\ldots,y_{in_i}\right)\in \mathbb{R}^{n_i}$, respectively, 
where $\mathbf{x}_{ij} \in \mathcal{X} := \mathbb{R}^M$ is a column vector and $n_i$ is the number of observations associated with the $i^{th}$ local model. 
We call $\mathcal{X}$ the \textit{input space}.

When creating the $i^{th}$ local GP, its \textit{center} is defined to be the centroid of $X_i$ 
as follows: 
\begin{equation}
    \mathbf{c}_i=\frac{1}{n_i}\sum_{j=1}^{n_i} \mathbf{x}_{ij} .
    \nonumber 
\end{equation}
Centers are critical to the assignment of observations to local GPs, as well as prediction, in the splitting GP model. Each time a new observation is assigned to a local model, its center is recomputed.

The \textit{kernel function} $k\!\!:\!\mathcal{X} \times \mathcal{X} \mapsto\mathbb{R}$ is a positive-definite, symmetric function. A vector $\boldsymbol{\theta}$ parameterizes the kernel, and is called the \textit{hyperparameter} of the GP model. For simplicity, we omit $\boldsymbol{\theta}$ from our notation. We use the term \textit{feature space} to refer to the reproducing kernel Hilbert space in which $k$ implicitly computes an inner product. 

We write $f \sim \mathcal{GP}(\mu,k)$ to say that the function $f$ is distributed as a GP with mean $\mu$ and kernel function $k$, and implicitly assume that the domain of $f$ is $\mathcal{X}$. The notation $f|X_i,Y_i$ denotes that $f$ is conditioned on the data $X_i,Y_i$. We use $\mathbf{x}^*$ to denote a \textit{test input} at which a prediction is to be made and $f^*|\mathbf{x}^*$ to be a GP conditioned on the test input, i.e. the posterior distribution.

\subsection{The splitting procedure}

To address modeling a potentially non-stationary latent function, each local GP model is split in a manner which centers the resulting child GPs on regions of different means. This is done by performing principal component analysis (PCA) on the training inputs associated with the \textit{parent model}. The first principal component gives the direction of most variance of the training inputs, which implies that the corresponding orthogonal hyperplane is the minimizer of within-cluster variance among all linear bisections, leading to Principal Direction Divisive Partitioning (PDDP) \citep{boley_principal_1998}. 

Two new GPs, which we call \textit{child GPs}, are then created, each of which is assigned the data of the parent model from one side of the hyperplane, as in Fig.~\ref{fig:pca_discontinuity}(a). This heuristic is based on the idea that, by fitting separate GPs to the subsets of training inputs which are maximally different, the model may best adapt to non-stationary behavior of the latent function. We utilize PDDP since the splitting procedure is computed efficiently in closed form, i.e. without the use of a convergent algorithm such as \textit{k}-means. It also allows for different choices in how the principal direction is computed. For example, in a setting where observations are fully sequential, the principal direction can be efficiently approximated using Oja's rule \citep{oja1997nonlinear} to avoid a singular value decomposition of the data matrix. 



\begin{figure}[h]
\centering
	{
		\subfigure[Splitting a 2-d input data set using PDDP into two subsets for two child GPs. Note the principal direction (orange vector), orthogonal hyperplane (green line), and the centroid ($\times$) of each subset colored in blue or red.]{\includegraphics[height=1.5in]{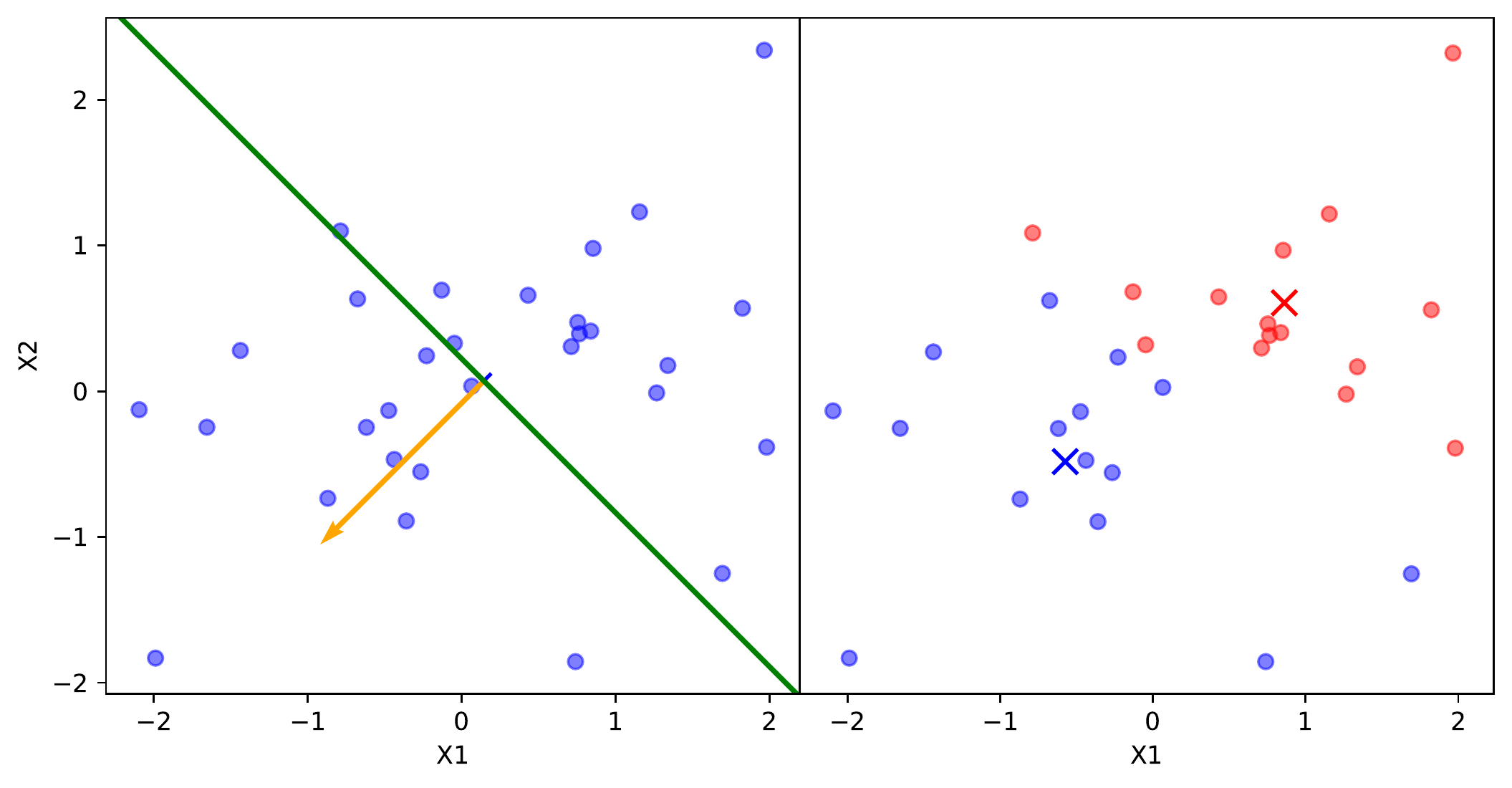}}
		\hfill
		\subfigure[The discontinuous predictive mean of $Y$ in $X$ is due to using only the closest child GP (left) instead of both child GPs (right). Note the center of each child GP (red vertical line).]{\includegraphics[height=1.5in]{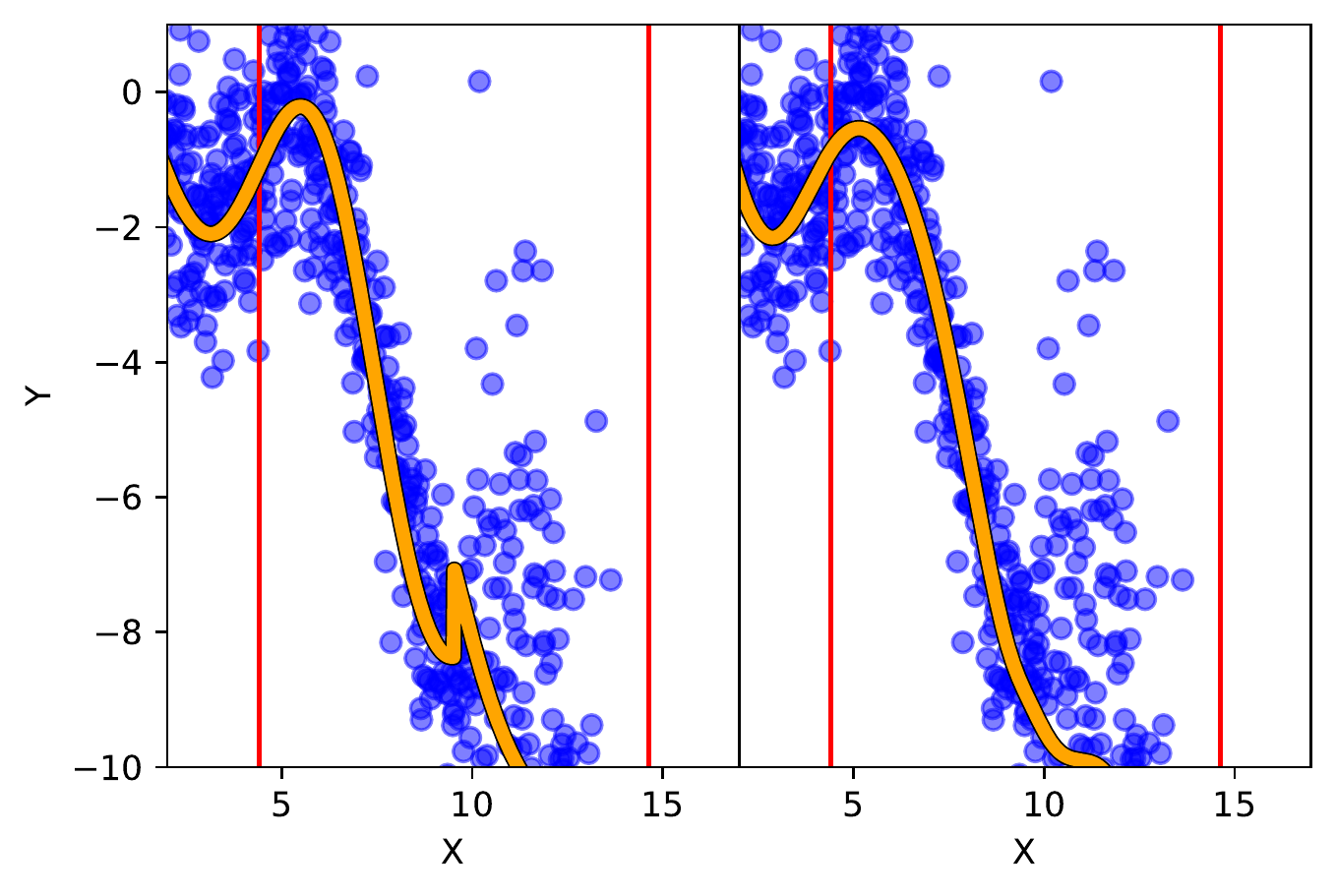}}
	}
	\caption{Illustration of how the proposed GP method uses PDDP for splitting in (a) and maintains continuous preditive mean in (b).} 
	\label{fig:pca_discontinuity}
\end{figure}

Critically, the prior of each child GP is taken to be the posterior of the parent model when conditioned on the \(m\) training observations, reflecting the Bayesian belief that the latent function's local covariance structure is not too dissimilar from its covariance structure in the larger subspace prior to splitting.

Formally, in the function space inference view of GP regression,
\begin{equation}
    f_{child} \sim f_{parent}|X_{parent},Y_{parent}
    \label{child_prior},
\end{equation}
where \(f_{child}\) is the prior of the child and $f_{parent}|X_{parent},Y_{parent}$ is the posterior of the parent given the training data \(X_{parent},Y_{parent}\). 
This assumption of local similarity is implicit in the use of many smooth kernel functions, particularly the radial basis function family, which are infinitely differentiable.

\subsection{Aggregating predictions of local models}

A new data point \(\left(\mathbf{x},y\right)\) is assigned to a child GP whose center is most similar to the predictor \(\mathbf{x}\) in the feature space, as determined by the kernel function. That is, for \(C\) child GPs, the data point is assigned to the child GP indexed by \[i_{assign}:=\underset{i=1,...,C}{\arg\!\max}\,k(\mathbf{c}_i,\mathbf{x}).\]
Similarly, the posterior mean at the test input \(\mathbf{x}^*\) is computed by weighting the prediction of each child GP by the relative similarity of its center to the test input in the feature space. This idea is based upon an interpretation of the predictive mean as weighting observations by their similarity to the point of prediction in the feature space, according to \citet{shen2006fast}. The same idea was also used by \citet{nguyen-tuong_local_2009}.

Unlike these previous works, we take a weighted average of \textit{all} predictions of child GPs, rather than only a subset, as follows: 
\begin{equation}
   \mathbb{E}[\,f^*|\mathbf{x}^*\,]=S^{-1}\sum_{i=1}^{C} k(\mathbf{c}_i,\mathbf{x}^*)\mathbb{E}[\,f^*_i|X_i,Y_i,\mathbf{x}^*\,]\label{prediction_eq},
\end{equation}
where \(
    S=\sum_{i=1}^{C} k(\mathbf{c}_i,\mathbf{x}^*)\label{s_eq}.
\)
This may be interpreted as weighting each child GP's mean prediction by the similarity of its local region to the test input \(\mathbf{x}^*\). In the next section, we will show that using predictions from \textit{all} child GPs has important consequences on the theoretical properties of the resulting model.


\section{Theoretical properties of the splitting GP model}
\subsection{Continuity of the splitting GP}
As a consequence of this weighting procedure and the choice of priors for the child GPs, the splitting GP model has some convenient properties. The proofs of the following propositions may be found in the appendix.

\begin{proposition}
\label{splitting}
Let \(f^*_{parent}|X_{parent},Y_{parent} \sim \mathcal{GP}(\mu,k)\) be the full GP prior to splitting and let \(\mathbf{x}^*\) be a test input. The child GPs \(f_i|X_i,Y_i \sim \mathcal{GP}(\mu_i,k),\,i=1,2\) from the first split have the property that, prior to being updated with any new observations, \(f^*|\mathbf{x}^*=f^*_{parent}|X_{parent},Y_{parent},\mathbf{x}^*\). That is, the predictive distribution is preserved by the splitting procedure.
\end{proposition}

This result is somewhat intuitive, since no additional evidence has been obtained which might alter our Bayesian beliefs about the latent function; we have simply altered the model structure. In agreement with this idea, the equality of the parent and childrens' predictive distribution no longer holds after any new data is assigned to one of the child GPs. Note that, in general, this relationship does not hold for any split after the first.

An important property of the splitting GP model is the preservation of continuity properties of the child GPs. We provide a result which shows the necessary and sufficient condition for continuity of the splitting GP model's predictions.

\begin{proposition}{}
\label{ms_continuity}
Suppose \(k\) is a kernel function and \(f_i|X_i,Y_i \sim \mathcal{GP}(\mu_i,k),\,i=1,...,C\). Then the random field given by \[f^*|\mathbf{x}^*=S^{-1}\sum_{i=1}^{C} k(\mathbf{c}_i,\mathbf{x}^*)\,f^*_i|X_i,Y_i,\mathbf{x}^*\] is mean square continuous in the input space if and only if the kernel function, $k$, is continuous. Under the same condition, the mean prediction \(\mathbb{E}[\,f^*|\mathbf{x}^*\,]\) is also continuous in the input space.
\end{proposition}{}

While it may seem obvious that the mean prediction is continuous in the input space, it is reassuring to know that the random field created by the splitting model has the same sufficient condition for mean square continuity as the underlying child GPs. Intuitively, we are computing a continuously weighted average of smooth random fields, so the resulting random field should also be smooth.

Note that it is not necessary to aggregate the predictions of each local model. Instead, one may elicit predictions only from the \(C_0<C\) local models most similar to the point of prediction, as measured by the kernel. This variation of the prediction method, as used in \citep{nguyen-tuong_local_2009}, may result in faster computation of predictions. However, we caution against this practice without careful consideration. This prediction method will result in a loss of continuity of the mean prediction \(\mathbb{E}[\,f^*|\mathbf{x}^*\,]\), and consequently the mean square continuity of the random field \(f^*|\mathbf{x}^*\), since the mean prediction is computed using a maximum function, which is not continuous. The discontinuity will manifest as sharp jumps in the mean predictions, as illustrated in Fig.~\ref{fig:pca_discontinuity}(b).

\subsection{Complexity analysis of the algorithm}

The algorithm for the splitting GP model has improved complexity in both memory and training/prediction time compared to the full GP regression. Additionally, it maintains some benefit in asymptotic complexity relative to other local GP methods.

The complexity of updating the splitting GP model with a single datum is bounded above by \(\mathcal{O}(m^3)\), corresponding to a matrix inversion of one child GP. The contribution of the PCA-based splitting procedure to the time complexity is negligible. Assuming the PCA is performed via naive singular value decomposition, the procedure would have \(\mathcal{O}(m^2)\) complexity amortized over $m$ sequential observations, which yields a linear additive term $m<m^3$. 

While each child GP has at most \(m\) observations, the average case will clearly be lower. It should be noted that we explicitly chose to not utilize a rank-one update of the Cholesky decomposition of the kernel matrix during the update procedure, a method which is used in other local GP methods such as in \cite{nguyen-tuong_local_2009}. We made this choice to permit updating the model with a batch of observations, in addition to fully sequential updating, which would require an update of rank greater than one. We also observed empirically that repeatedly updating a single child GP using rank-one updates would cause numerical issues if the kernel length-scale parameter is small.

Since the time complexity of the algorithm is characterized primarily by the parameter \(m\), the largest number of observations which may be associated with a single child model, it is straightforward to select an appropriate parameter value for applications requiring rapid sequential updating. After empirically determining the necessary wall-clock time needed for an update, the parameter may be adjusted appropriately.

The splitting algorithm also imposes a lower memory complexity in terms of the number of observations, $n$. It is well known that local GP methods effectively store only a block-diagonal approximation of the full covariance matrix \citep{park_patchwork_2018,blockgp}. In particular, when the splitting algorithm has \(n\) observations, $\lfloor n/m \rfloor$ child GPs have been created, each of which will store a kernel matrix with up to \(m^2\) entries. The resulting memory complexity of the algorithm is then \( \mathcal{O}(mn) < \mathcal{O}(n^2) \), where \(\mathcal{O}(n^2)\) is the memory complexity of a full GP regression. In contrast, the rBCM by \cite{deisenroth_distributed_2015} has an asymptotic memory complexity of \(\mathcal{O}(n^2/E)\), where \(E\) is the (constant) number of experts specified as a parameter.  \citet{nguyen-tuong_local_2009} did not present the memory complexity of their local GP model but, under the mild assumption that the input space is bounded, it can be shown that the asymptotic memory complexity is \(\mathcal{O}(n^2)\). Notably, the splitting GP model is, to the best of our knowledge, the only local GP model to achieve a linear memory complexity.

\section{Experiments}

The efficacy of the splitting GP model was experimentally evaluated on one synthetic and two real-world data sets. For each experiment, all models used the radial basis function kernel, with automatic relevance determination \citep{radford_ard} enabled. We compared the performance of the splitting GP model to the local GP regression of \citet{nguyen-tuong_local_2009}, the robust Bayesian committee machine (rBCM) of \citet{deisenroth_distributed_2015}, and full GP regression as a baseline comparison. Each of these models was chosen since they may be updated using sequential data, and make no use of a ``complete'' training data set to inform the model. This is in contrast to, for example, the patchwork kriging method of \citet{park_patchwork_2018}, which utilizes information from the entirety of data for domain decomposition. 

\begin{wrapfigure}{L}{.3\linewidth}
    \centering
    \includegraphics[width=\linewidth]{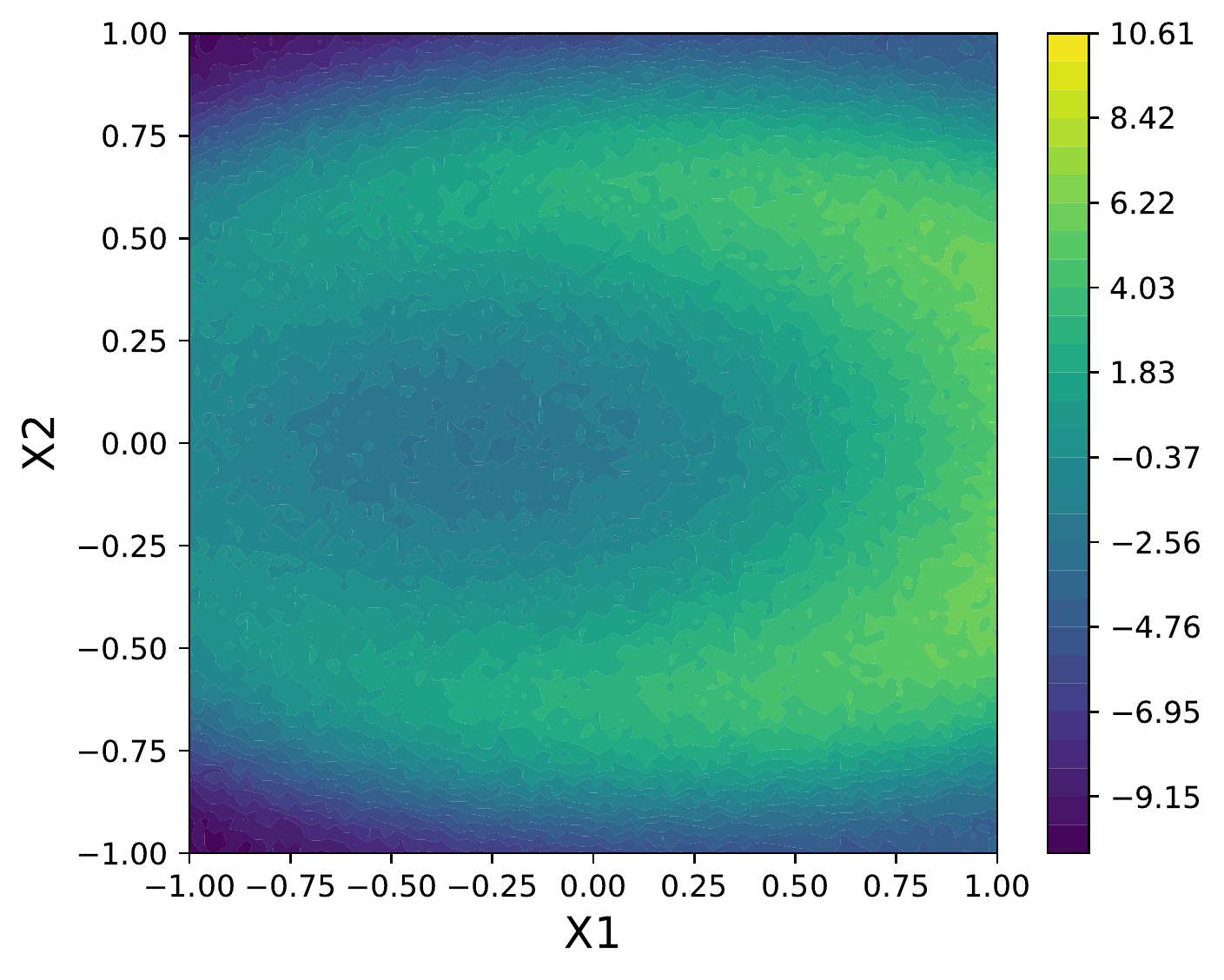}
    \caption{The contours of the response of the synthetic data set from \eqref{synthetic_response_eq}.}
    \label{fig:synthetic_data}\vspace{-1em}
\end{wrapfigure}{}

\subsection{Synthetic data with a non-stationary latent function}
The synthetic data set (see Fig.~\ref{fig:synthetic_data}) was constructed to have a non-stationary latent function $f$ of a two-dimensional predictor $(x_1,x_2)\in [-1,1]^2$. The response function is defined as 
\begin{align*}
y&=f(x_1,x_2)+\epsilon\\
&=5\sin(x_1^2+x_2^2)+3x_1+\epsilon, \stepcounter{equation}\tag{\theequation}\label{synthetic_response_eq}
\end{align*} 
where $\epsilon\overset{iid}{\sim}\mathcal{N}(0,\sigma)$ with $\sigma=(0.05)\underset{x_1,x_2}{\max}\,f(x_1,x_2)$.
To construct the synthetic data set, a grid of 10,000 points was constructed in the $x_1$-$x_2$ plane and the latent function evaluated at each point. A subset of 2500 observations were sampled uniformly, without replacement, from the resulting grid for each replicate of the experiment. This sampling over the grid ensured that no two observations are too close to one another, which may cause numerical issues during training.

For this experiment, we compared three metrics of the models' performance: prediction error (in mean squared error, or MSE), memory usage (in kilobytes, or kB), and training time (in seconds). Each metric was estimated using a 5-fold cross validation procedure. We utilized common random numbers \citep{CRN} as a variance reduction technique. Each experiment was replicated 10 times to reduce variability of results. In Fig.~\ref{fig:synthetic}, the shaded region shows the 95\% pointwise confidence region for the mean metric. The experiment was performed with different numbers of observations, ranging from 100 to 2500 in increments of 100, to demonstrate the relative data requirements for each model to converge. Each model was updated sequentially with single observations to simulate a streaming data setting.

For each alternative model compared, the experiment was replicated for a range of parameter values and the most favorable, in terms of MSE, results for each model are reported. For the splitting GP model, we used a splitting limit of $m=500$. For the rBCM, $E=10$ local experts were used. We found the parameter $w_{gen}$ of the local GP model difficult to tune and eventually used $w_{gen}=10^{-3}$ based on an extensive grid search (detailed in the appendix).

In Fig.~\ref{fig:synthetic}, one can see that the splitting model and full GP required relatively few observations to achieve strong predictive power. The rBCM was much slower to converge, and exhibited high variability in its MSE across replicates of the experiment. We attribute the greater variability in MSE to the rBCM's random assignment of data to GP experts. On the other hand, the splitting GP model exhibited remarkably low variability in MSE between replicates. It is worth noting that the splitting model's MSE slowly increased as it splits at intervals of 500 observations (see Fig.~\ref{fig:synthetic_zoom}).

The memory usage of the splitting GP model proved to be significantly lower than the full GP model, and marginally higher than that of the rBCM. However, for \textasciitilde2500 observations, it can be seen that the memory usage of the rBCM model began to surpass that of the splitting GP model. This is to be expected, since the asymptotic memory complexity of the rBCM is quadratic, as opposed to the linear complexity of the splitting GP model.

The training time of the rBCM and local GP models were found to be comparable, and the splitting GP model took slightly longer. The full GP took significantly longer to train. The change in regime of the training time of the full GP at \textasciitilde1100 observations is due to specialized numerical methods in the \texttt{GPyTorch} library \citep{gardner_gpytorch_2018}.

\begin{figure}
    \centering
    \subfigure[MSE]{\includegraphics[trim = .1in 0 0in 0in, clip,width=.32\linewidth]{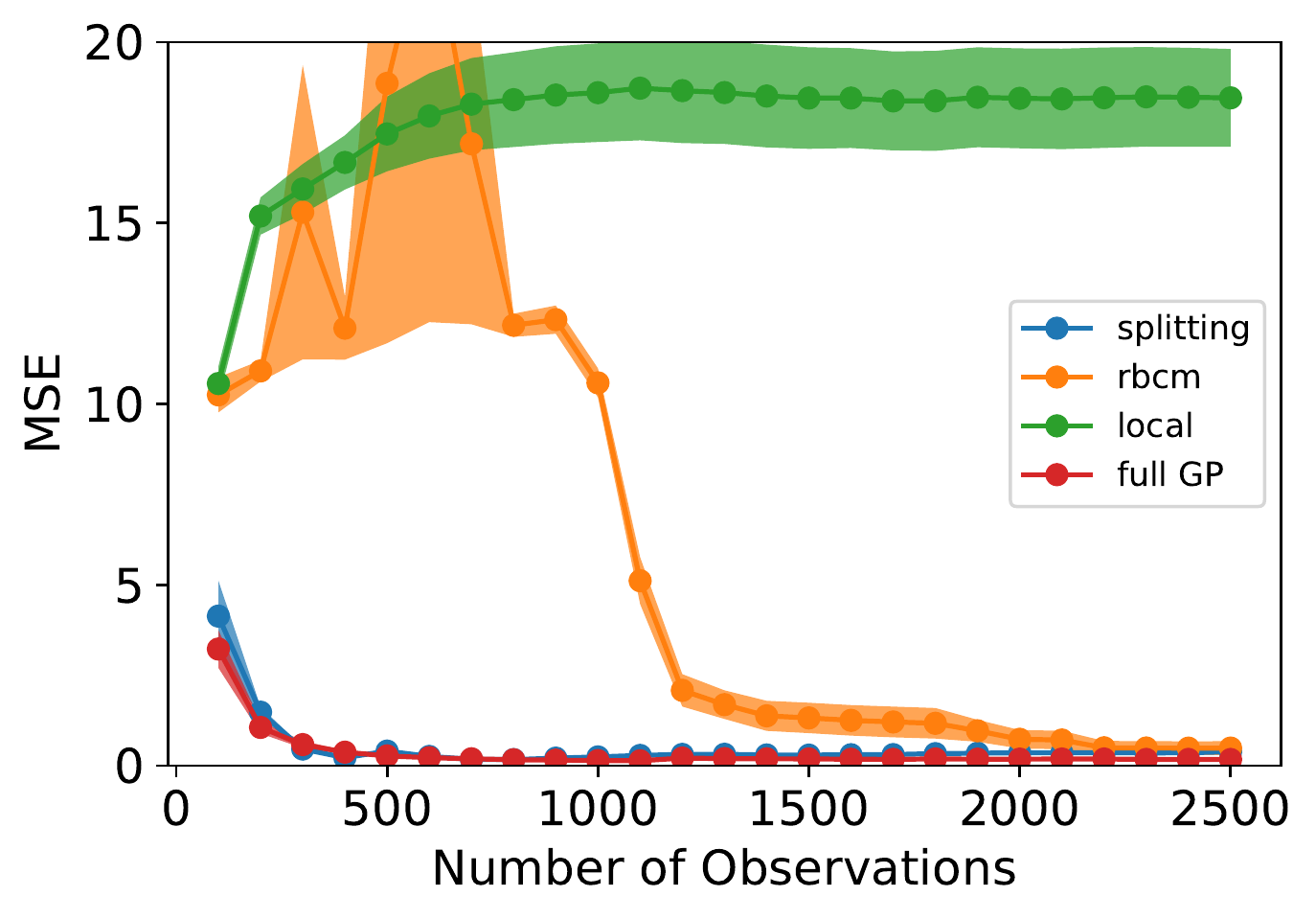}}
    \hfill
    \subfigure[Memory usage]{\includegraphics[trim = .1in 0 0in 0in, clip,width=.32\linewidth]{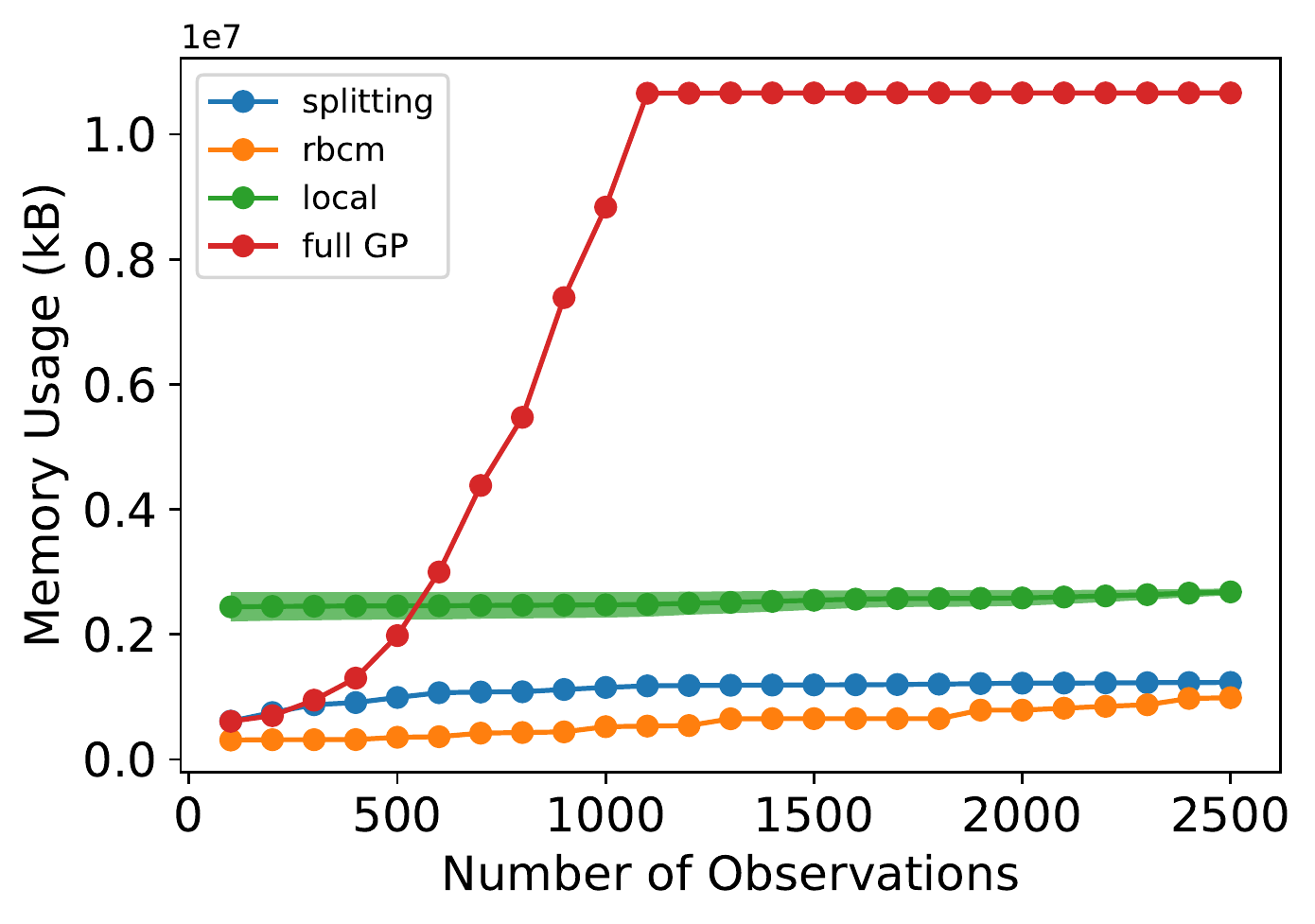}}
    \subfigure[Training time]{\includegraphics[trim = .1in 0 0in 0in, clip,width=.32\linewidth]{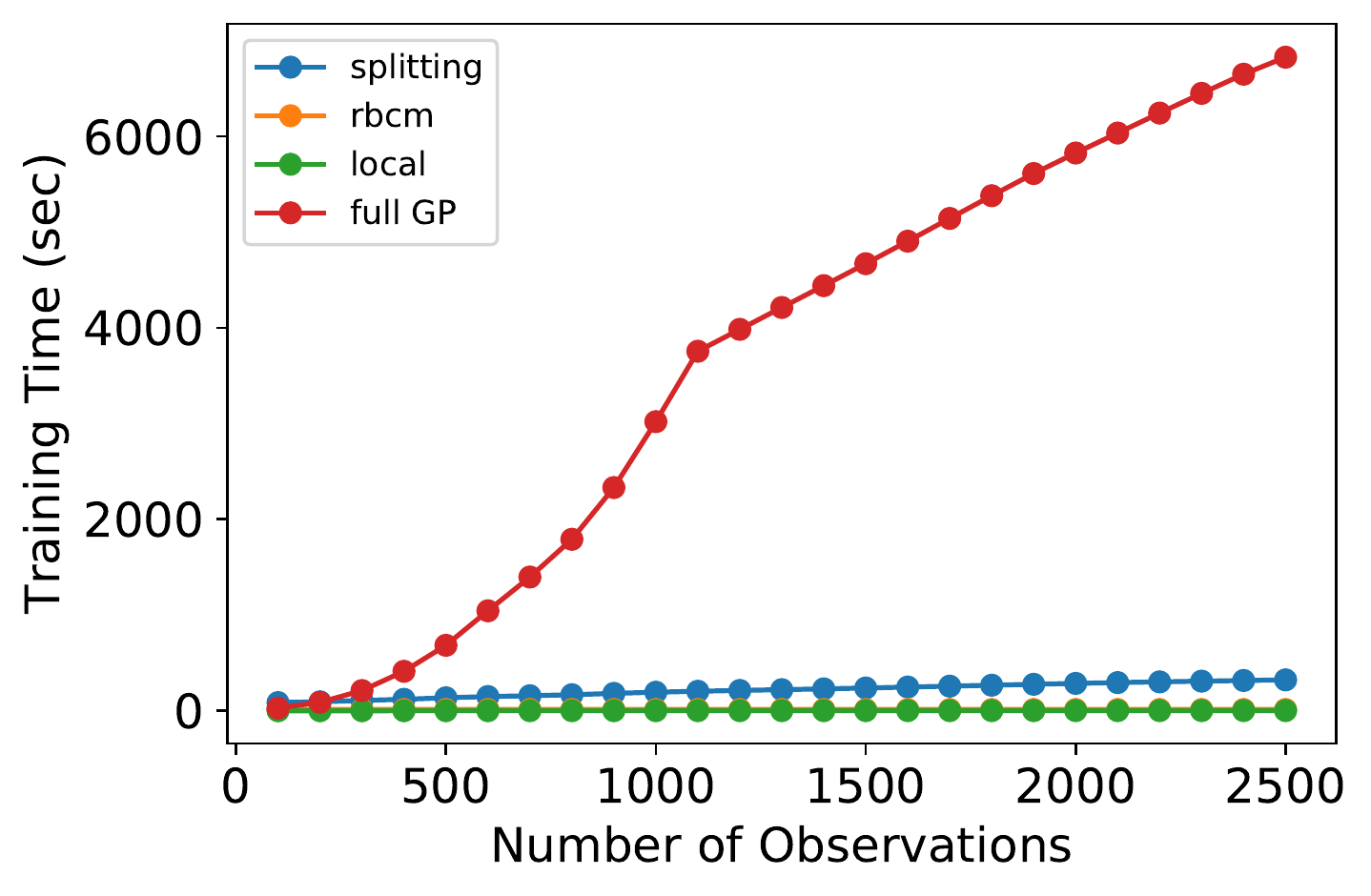}}
    \hfill
    \caption{Comparison of alternative models' MSE, memory usage, and training time for the synthetic data set from \eqref{synthetic_response_eq}. The shaded areas give the 95\% pointwise confidence region for each model.}
    \label{fig:synthetic}
\end{figure}

\begin{wrapfigure}{R}{.3\linewidth}
    \centering
    \includegraphics[trim = .1in 0 .5in .4in, clip,width=\linewidth]{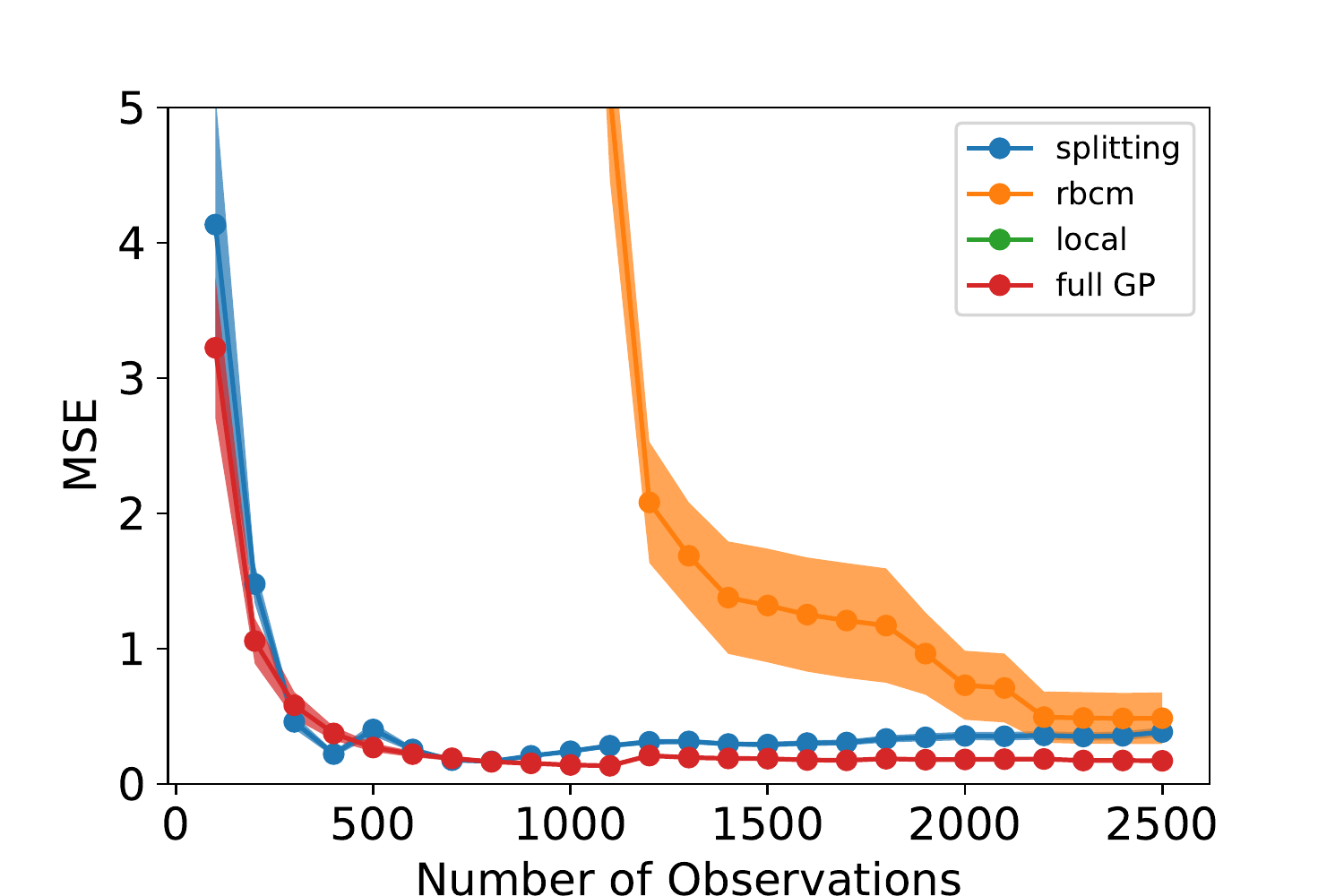}
    \caption{Zoom-in of Fig.~\ref{fig:synthetic}(a). The MSE of the splitting GP model slowly increases as it splits to maintain an efficient approximation of the full GP.}
    \label{fig:synthetic_zoom}
\end{wrapfigure}

\subsection{Real-world data in robotic control application}
The first real-world data set, \textit{kin40k}, consists of 40,000 records describing the location of a robotic arm as a function of an 8-dimensional control input \citep{nguyen_kin40k_data}. This data set was chosen since it exemplifies a task where the fast sequential updating and low memory profile offered by the splitting GP model are desirable. Furthermore, \textit{kin40k} is a popular benchmark for other GP regression methods \citep{deisenroth_distributed_2015,nguyen2014fast,lazaro-gredilla_sparse_2010}, and thus facilitates direct comparison.

We used the same training/test split of 10,000/30,000 points as in the previous work \citep{deisenroth_distributed_2015,nguyen2014fast,lazaro-gredilla_sparse_2010}. Observations were added to the models in batches, with each batch having a number of observations equal to the number of observations per local model. The results can be seen in Fig.~\ref{fig:real-world}(a). 
The splitting model's root-mean-square error (RMSE) is comparable with that of the rBCM, which is designed under a stronger assumption (that the latent function is well-modeled by a single GP), which is satisfied in this stationary regression task.  
In contrast, the local GP model struggled to achieve the same RMSE in this experiment (see the appendix for more detail on the difficulty tuning the parameter $w_{gen}$).

\begin{figure}[t!]
\centering
	{
		\subfigure[\textit{kin40k} data set.]{\includegraphics[height=1.8in]{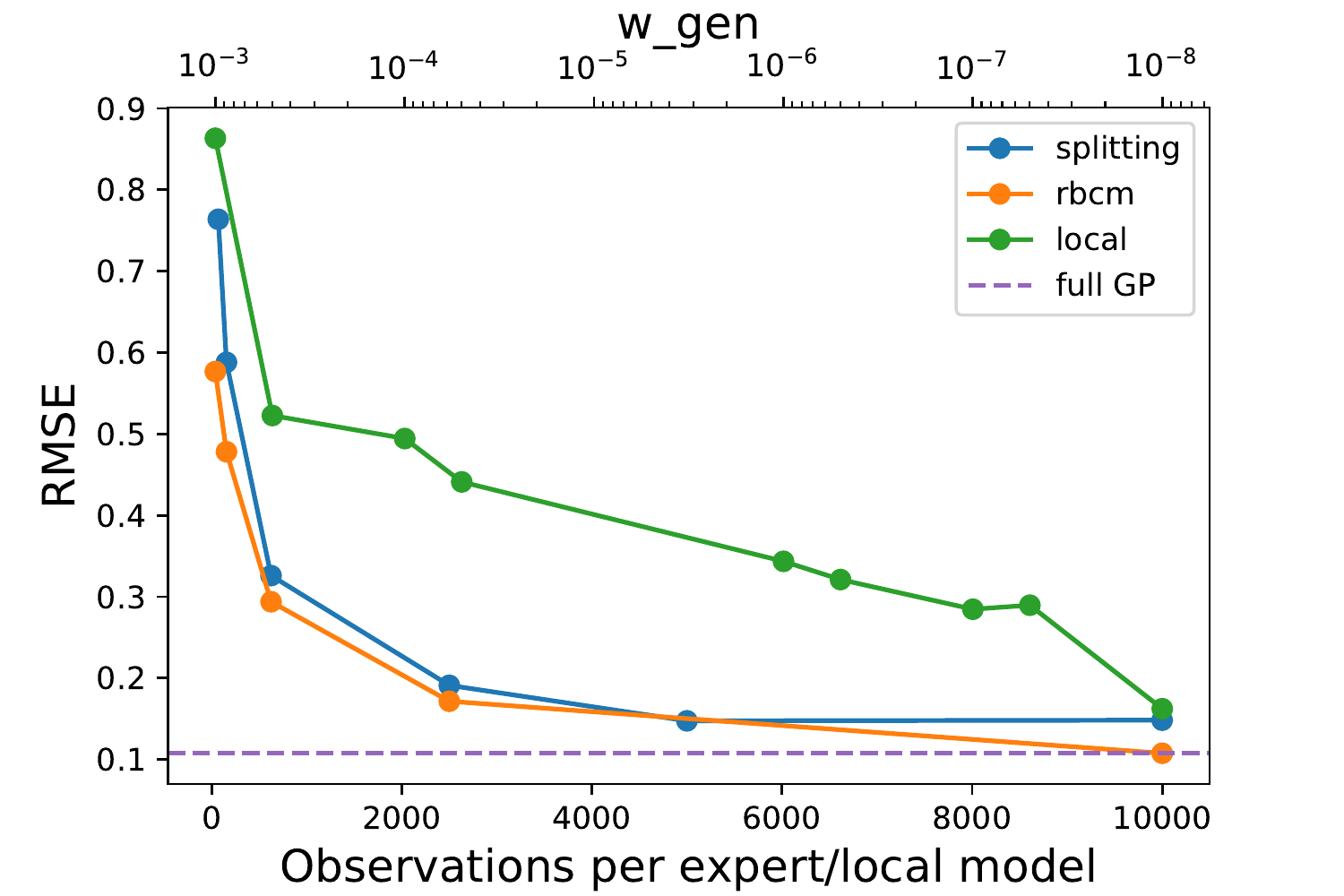}}
		\subfigure[\textit{powergen} data set.]{\includegraphics[height=1.8in]{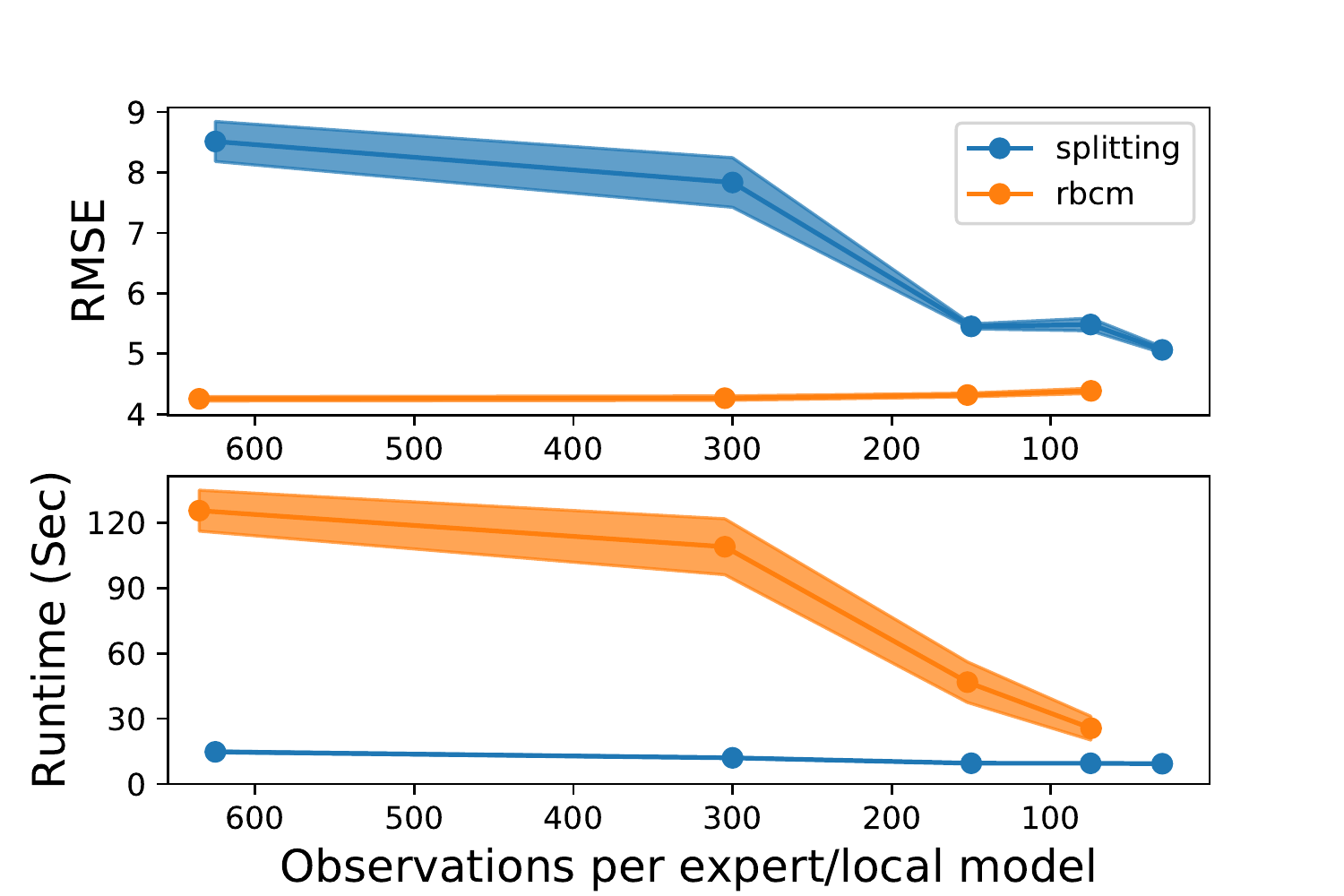}}		\caption{Comparison of alternative models for two real-world datasets. The bottom label on the $x$-axis for both plots gives the parameter values of the splitting GP model (observations per local model, $m$) and the rBCM (observations per expert). (a): The top label on the $x$-axis gives the value of the parameter \(w_{gen}\)  in the local GP model of \citet{nguyen-tuong_local_2009}. (b): Markers on the line plots show the average over 10 replicates, and the shaded regions denote the 95\% pointwise confidence intervals. For the final parameter value of 30 observations per expert, numerical errors occurred with the rBCM model, so no result is shown.}
		\label{fig:real-world}
	}
\end{figure}


\subsection{Real-world data in power plant control application}
 The second real-world data set is \textit{powergen}, which consists of four predictors describing control inputs of a combined cycle power plant, and a response of the net electrical energy production. This data set is due to \citet{kaya2012local} and \citet{tufekci2014prediction} and is publicly available \citep{Dua:2019}. We pre-processed the data to remove duplicate observations, leaving a total of 7622 observations.
 
 In the \textit{powergen} experiment, we compared both the predictive error (in RMSE) and the combined training and prediction time (in seconds) of the splitting GP model and rBCM. We compared only these two models, since they proved to be the most competitive local GP methods in earlier experiments. The data set was randomly divided into a 80\%/20\% train/test split, and each model evaluated for several different parameter values. Observations were added to the models in batches, with each batch having a number of observations equal to the number of observations per local model. The experiment was replicated 10 times, and common random numbers were used for splitting the data set and training the models for sharper comparisons between the models. The results can be seen in Fig.~\ref{fig:real-world}(b). 
 

The splitting GP model's RMSE decreased with the number of observations per local model, achieving a comparable performance with the rBCM.  
In terms of runtime, the splitting GP model was significantly faster, while having little variability between replicates. While the rBCM's runtime decreased with the number of observations per expert, 
its average runtime was 2.5-6 times longer than that of the splitting GP model, depending on the choice of parameters. 

\section{Conclusion}

In this paper, we have developed an algorithm for constructing splitting GP regression models for potentially non-stationary latent functions using streaming data. We have shown that splitting GP models attain comparable predictive performance, while addressing critical shortcomings of other local GP models, such as discontinuity of predictions, lack of flexibility for modeling non-stationary latent functions, and opaque parameters which may be challenging to tune. Furthermore, splitting GP models are shown to enjoy \textit{linear} memory complexity which, to the best of our knowledge, is the best among existing local GP methods, which typically have quadratic memory complexity. An implementation of the splitting GP model is available in the supplementary material to this article. 

\bibliography{GPPapers.bib}
\bibliographystyle{plainnat}

\end{document}


\maketitle

\textbf{Note}: In this supplementary material, we refer to equations and figures in the main paper.

\section{Proofs}

\begin{proposition}
Let \(f^*_{parent}|X_{parent},Y_{parent} \sim \mathcal{GP}(\mu,k)\) be the full GP prior to splitting and let \(\mathbf{x}^*\) be a test data point. The child GPs \(f_i|X_i,Y_i \sim \mathcal{GP}(\mu_i,k),\,i=1,2\) from the first split have the property that, prior to being updated with any new observations, \(f^*|\mathbf{x}^*=f^*_{parent}|X_{parent},Y_{parent},\mathbf{x}^*\). That is, the predictive distribution is preserved by the splitting procedure.
\end{proposition}

\begin{proof}
We consider the case of a single local model being split into two new local models. The posterior mean at an input \(\mathbf{x}^*\) is given by
\begin{align*}
   f^*|\mathbf{x}^*& =S^{-1}k(\mathbf{c}_1,\mathbf{x}^*)f^*_1|X_1,Y_1,\mathbf{x}^*+S^{-1}k(\mathbf{c}_2,\mathbf{x}^*)f^*_2|X_2,Y_2,\mathbf{x}^* & \text{by definition in }\eqref{prediction_eq}\\
   & =\alpha f^*_1|X_1,Y_1,\mathbf{x}^*+(1-\alpha) f^*_2|X_2,Y_2,\mathbf{x}^*,\quad
   \alpha=S^{-1}k(\mathbf{c}_1,\mathbf{x}^*) & \text{by } \eqref{s_eq}\\
   & =\alpha f^*_1|\mathbf{x}^*+(1-\alpha) f^*_2|\mathbf{x}^*,\,\,
   \alpha=S^{-1}k(\mathbf{c}_1,\mathbf{x}^*)\\ & \text{by assumption that no additional data has been observed since the split} \\
   & =\alpha f^*_{parent}|X_{parent},Y_{parent},\mathbf{x}^*+(1-\alpha) f^*_{parent}|X_{parent},Y_{parent},\mathbf{x}^* & \text{by } \eqref{child_prior}\\
   & = f^*_{parent}|X_{parent},Y_{parent},\mathbf{x}^* .
\end{align*}{}
\end{proof}{}

\begin{proposition}{}
Suppose \(k\) is a kernel function and \(f_i|X_i,Y_i \sim \mathcal{GP}(\mu_i,k),\,i=1,...,C\). Then the random field given by \[f^*|\mathbf{x}^*=S^{-1}\sum_{i=1}^{C} k(\mathbf{c}_i,\mathbf{x}^*)\,f^*_i|X_i,Y_i,\mathbf{x}^*\] is mean square continuous in the input space if and only if the kernel function, $k$, is continuous. Under the same condition, the mean prediction \(\mathbb{E}[\,f^*|\mathbf{x}^*\,]\) is also continuous in the input space.
\end{proposition}{}

\begin{proof}
A common result on random fields gives that the random field \(f^*|\mathbf{x}^*\) is mean square continuous if and only if its expectation and covariance functions are continuous; see \citet{hristopulos_geometric_2020}, Theorem 5.2, for example. It then suffices to show that \(\mathbb{E}[\,f^*|\mathbf{x}^*\,]\) is continuous. 

\begin{align}
    \mathbb{E}\mathbb[\,f^*|\mathbf{x}^*\,] & = \mathbb{E}\!\left[\,S^{-1}\sum_{i=1}^{C} k(\mathbf{c}_i,\mathbf{x}^*)\,f^*_i\middle\vert X_i,Y_i,\mathbf{x}^*\,\right]\\
    & = S^{-1}\sum_{i=1}^{C} k(\mathbf{c}_i,\mathbf{x}^*)\mathbb{E}[\,f^*_i|X_i,Y_i,\mathbf{x}^*\,]
\end{align}
Recall that the predictive mean \(\mathbb{E}[\,f^*_i|X_i,Y_i,\mathbf{x}^*\,]\) is a linear function of \(\mathbf{x}^*\), and therefore continuous. Note from \eqref{s_eq} that \(S^{-1}\) is continuous if and only if \(k\) is continuous. In this case, \(\mathbb{E}[\,f^*|\mathbf{x}^*\,]\) is continuous, and hence \(f^*|\mathbf{x}^*\) is mean square continuous.
\end{proof}{}

\section{The splitting GP algorithm}
Here we specify three algorithms which describe the main operations of the splitting GP model: splitting a GP, updating the model, and computing the predictive mean. We make use of the much of the same notation defined in the main paper.

From a computational perspective, the necessary components to construct a GP model are the training inputs and training response, $X_i$ and $Y_i$, respectively. For convenience of notation, the following algorithms are thus described in terms of matrix operations on these variables. We will use the triple $\left( X_i,Y_i,\mathbf{c}_i \right)$ to characterize the $i^{th}$ child GP, where $\mathbf{c}_i$ is the center of the GP as defined in the main paper. The entirety of the splitting GP model is itself given as a triple $\left(\mathcal{A},k_{\boldsymbol{\theta}},m \right)$, where $\mathcal{A}=\{ \left( X_i,Y_i,\mathbf{c}_i \right) : i=1,\dots,C \}$ is a set of the child GPs, $k_{\boldsymbol{\theta}}$ is the kernel function with hyperparameter $\boldsymbol{\theta}$, and $m$ is the splitting limit of the splitting GP model. \textit{The following pseudo-code makes use of explicit \texttt{for} loops for clarity, but please note that our implementation (see the supplementary material)  makes use of vectorized versions of these algorithms for efficiency.}

We first specify the splitting algorithm, which is used to divide a GP into two smaller child GPs. To keep the pseudo-code as general as possible, we define the function \texttt{PrincipalDirection}$(X)$ as one which computes the first principal component vector of a matrix $X$. 

\begin{algorithm}
\SetKwFunction{PrincipalDirection}{PrincipalDirection}
\SetAlgoLined
 $\hat{X}_{1},\hat{X}_{2} \leftarrow \left[\quad\right],\left[\quad\right]$\;
 $\hat{Y}_{1},\hat{Y}_{2} \leftarrow
 \left[\quad\right],\left[\quad\right]$\;
 $\hat{n}_{1},\hat{n}_{2} \leftarrow 0,0$\;
 $\mathbf{v} \leftarrow \PrincipalDirection(X_i)$\;
 \For{$j \leftarrow 1$ \KwTo $n_i$}{
  $I \leftarrow \mathbf{v}^T\left(\mathbf{x}_{ij}-\mathbf{c}_i\right)$\;
  \eIf{$I > 0$}{
   $\hat{X}_{1} \leftarrow \left[\hat{X}_{1}^T \quad \mathbf{x}_{ij} \right]^T$\;
   $\hat{Y}_{1} \leftarrow \left[\hat{Y}_{1}^T \quad y_{ij} \right]^T$\;
   $\hat{n}_1 \leftarrow \hat{n}_1 + 1$\;
   }{
   $\hat{X}_{2} \leftarrow \left[\hat{X}_{2}^T \quad \mathbf{x}_{ij} \right]^T$\;
   $\hat{Y}_{2} \leftarrow \left[\hat{Y}_{2}^T \quad y_{ij} \right]^T$\;
   $\hat{n}_2 \leftarrow \hat{n}_2 + 1$\;
  }
 }
 $\hat{\mathbf{c}}_1 \leftarrow \hat{n}_1^{-1}\; \sum_{j=1}^{\hat{n}_1} \hat{\mathbf{x}}_{1j}$\;
 $\hat{\mathbf{c}}_2 \leftarrow \hat{n}_2^{-1} \sum_{j=1}^{\hat{n}_2} \hat{\mathbf{x}}_{2j}$\;
 \caption{\texttt{Split}$(\left( X_i,Y_i,\mathbf{c}_i \right))$}
 \KwResult{$\left( \hat{X}_1,\hat{Y}_1,\hat{\mathbf{c}}_1 \right),\left( \hat{X}_2,\hat{Y}_2,\hat{\mathbf{c}}_2 \right)$}
\end{algorithm}

Using \texttt{Split}, we can now define the algorithm \texttt{Update} for updating the splitting GP model given a new data point $\left(\mathbf{x},y\right)$. We use the shorthand \texttt{Train} to mean the standard fitting of a GP model with the hyperparameter $\boldsymbol{\theta}$ to data by means of maximizing the log marginal likelihood \citep{Gaussian_Processes_for_ML}. 

\begin{algorithm}[H]
\SetKwFunction{PrincipalDirection}{PrincipalDirection}
\SetKwFunction{Split}{Split}
\SetKwFunction{Train}{Train}
\SetAlgoLined
\eIf{$C=0$}{
    $X_1 \leftarrow \mathbf{x}^T$\;
    $Y_1 \leftarrow y^T$\;
    $\mathbf{c}_1 \leftarrow \mathbf{x}$\;
    $\mathcal{A} \leftarrow \left[ \left( X_1,Y_1,\mathbf{c}_1 \right) \right]$\;
    $C \leftarrow 1$\;
}{
$I \leftarrow \underset{i=1,\dots,C}{argmax}\,\,k_{\boldsymbol{\theta}}(\mathbf{x},\mathbf{c}_i)$\;
$X_I \leftarrow \left[X_{I}^T \quad \mathbf{x} \right]^T$\;
$Y_I \leftarrow \left[Y_I^T \quad y \right]^T$\;
$n_I \leftarrow n_I + 1$\;
$\mathbf{c}_I \leftarrow n_I^{-1} \sum_{j=1}^{n_I} \mathbf{x}_{Ij}$\;
\If{$n_I>m$}{
    $\left( X_{I},Y_{I},\mathbf{c}_{I} \right),\left( X_{C+1},Y_{C+1},\mathbf{c}_{C+1} \right) \leftarrow \Split\left( \left( X_I,Y_I,\mathbf{c}_I \right) \right)$\;
    $\mathcal{A} \leftarrow \{ \left( X_i,Y_i,\mathbf{c}_i \right) : i=1,\dots,C+1 \}$\;
    $C \leftarrow C + 1$\;
}
}
\Train($\boldsymbol{\theta},\mathcal{A}$)\;
 \caption{\texttt{Update}$(\left(\mathcal{A},k_{\boldsymbol{\theta}},m \right),\left(\mathbf{x},y\right))$}
 \KwResult{$\left(\mathcal{A},k_{\boldsymbol{\theta}},m \right)$}
\end{algorithm}

Finally, the algorithm for computing the mean prediction of response at the input $\mathbf{x}$ follows. Here we use the abbreviated \texttt{Mean} function to give the posterior mean for a child GP $\left( X_i,Y_i,\mathbf{c}_i \right)$. The posterior mean for each child GP is computed in closed form using linear algebra; see \citep{Gaussian_Processes_for_ML}. 

\begin{algorithm}
\SetKwFunction{PrincipalDirection}{PrincipalDirection}
\SetKwFunction{Split}{Split}
\SetKwFunction{Mean}{Mean}
\SetAlgoLined
\For{$i \leftarrow 1$ \KwTo $C$}{
    $s_i \leftarrow k_{\boldsymbol{\theta}}(\mathbf{x},\mathbf{c}_i)$\;
    $E_i \leftarrow \Mean(\left( X_i,Y_i,\mathbf{c}_i \right),\mathbf{x})$\;
}
$ S \leftarrow \sum_{i=1}^C s_i $\;
$\hat{y} \leftarrow S^{-1} \sum_{i=1}^C s_i E_i$\;
\caption{\texttt{Predict}$(\left(\mathcal{A},k_{\boldsymbol{\theta}},m \right),\mathbf{x})$}
\KwResult{$\hat{y}$}
\end{algorithm}

\section{Hyperparameter tuning for the local GP model}
In the experiment with the synthetic data set, we initially performed a grid search on the parameter \(w_{gen}\)  of the local GP model \citep{nguyen-tuong_local_2009} from .9 to .1 at increments of .05, and obtained the best results for $w_{gen}=.1$. However, the resulting MSE was much higher than the other models considered. In the interest of fair comparison, we conducted a second grid search ranging from .1 to $10^{-3}$ by increments of $5\times10^{-4}$, and the best parameter was found to be $10^{-3}$, for which we show the results in Fig.~\ref{fig:synthetic}.

We chose not to continue lowering  $w_{gen}$ further, since the local GP model's MSE may be expected to decrease monotonically with $w_{gen}$. The parameter \(w_{gen}\)  defines a \textit{similarity threshold}, such that if a new observation is sufficiently dissimilar from existing local models (i.e. \(k(\mathbf{x}^*,\mathbf{c}_i)\leq w_{gen},\, \forall i=1,...,C\)), a new local model will be created. If $w_{gen}=0$, then only one local GP will be created, so that the model reduces to a full GP. This behavior can be seen in Fig.~\ref{fig:real-world}, where we performed a grid search on $w_{gen}$ ranging as low as $10^{-10}$. For values less than $10^{-8}$, limitations to floating point precision caused the local GP model's prediction procedure to become numerically unstable, so Fig.~\ref{fig:real-world} does not include these parameter values.


\bibliography{GPPapers.bib}
\bibliographystyle{plainnat}